\documentclass[a4paper,UKenglish,cleveref, autoref, thm-restate]{lipics-v2021}
\hideLIPIcs

\usepackage[usenames,dvipsnames]{xcolor}
\definecolor{darkgreen}{rgb}{0.0, 0.5, 0.0}

\title{\algname: Streaming Graph Clustering with Multi-Stage Refinement}

\author{Adil {Chhabra}}{Heidelberg University, Germany }{adil.chhabra@informatik.uni-heidelberg.de}{https://orcid.org/0009-0009-5726-9389}{}
\author{Shai {Dorian Peretz}}{Heidelberg University, Germany}{dorian.peretz@stud.uni-heidelberg.de}{https://orcid.org/0009-0007-2456-9073}{}
\author{Christian {Schulz}}{Heidelberg University, Germany}{christian.schulz@informatik.uni-heidelberg.de}{https://orcid.org/0000-0002-2823-3506}{}

\authorrunning{A. Chhabra et al.} %

\Copyright{Adil {Chhabra}, Shai {Dorian Peretz}, and Christian {Schulz}} 



\category{} %

\relatedversion{} %

\supplement{}%

\acknowledgements{We acknowledge support by DFG grant SCHU 2567/5-1. 
	Moreover, we would like to acknowledge Dagstuhl Seminar 23331 on Recent Trends in Graph Decomposition.
}

\EventEditors{}
\EventNoEds{2}
\EventLongTitle{}
\EventShortTitle{SEA 2025}
\EventAcronym{}
\EventYear{}
\EventDate{}
\EventLocation{}
\EventLogo{}
\SeriesVolume{}
\ArticleNo{}

\usepackage[ruled]{algorithm}
\usepackage[noend]{algpseudocode}
\usepackage{amsmath}
\usepackage{xspace}
\usepackage{graphicx}
\usepackage{tikz}
\usepackage{booktabs}
\usepackage{multirow}
\usepackage{array}
\usepackage{graphicx}
\usepackage{colortbl}
\usepackage{tabularx}
\usepackage{longtable}
\usepackage{tcolorbox}

\usetikzlibrary{external}
\definecolor{shade1}{gray}{0.95} %
\definecolor{shade2}{gray}{0.99} %

\newcommand{\algname}{\textsc{CluStRE}}

\newcommand{\sodass}{\,:\,}
\newcommand{\setGilt}[2]{\left\{ #1\sodass #2\right\}}
\DeclareMathOperator*{\argmax}{argmax}

\def\MdR{\ensuremath{\mathbb{R}}}

\newcommand{\instance}[1]{\texttt{#1}}

\usepackage{etoolbox}

\begin{document}
\nolinenumbers

\maketitle

\begin{abstract}
We present \algname, a novel streaming graph clustering algorithm that balances computational efficiency with high-quality clustering using multi-stage refinement. Unlike traditional in-memory clustering approaches, \algname~processes graphs in a streaming setting, significantly reducing memory overhead while leveraging re-streaming and evolutionary heuristics to improve solution quality. Our method dynamically constructs a quotient graph, enabling modularity-based optimization while efficiently handling large-scale graphs. We introduce multiple configurations of \algname~to provide trade-offs between speed, memory consumption, and clustering quality. Experimental evaluations demonstrate that \algname~improves solution quality by 89.8\%, operates 2.6$\times$ faster, and uses less than two-thirds of the memory required by the state-of-the-art streaming clustering algorithm on average. Moreover, our strongest mode enhances solution quality by up to 150\% on average. With this, \algname~achieves comparable solution quality to in-memory algorithms, i.e. over 96\% of the quality of clustering approaches, including \textsc{Louvain}, effectively bridging the gap between streaming and traditional clustering methods. 
\end{abstract}

\section{Introduction}
\label{sec:intro}

Graph clustering or community detection is the problem of identifying densely connected regions of a graph. The potential applications of graph clustering are vast, as nearly all systems with interacting or coexisting entities can be represented as graphs. Common applications include gaining insights into voter behavior, the emergence of trends, terrorist group formation and recruitment~\cite{SCHAEFFER200727} or the natural partitioning of data records onto pages~\cite{diwanclustering}, as well as analyzing protein interactions~\cite{lealprotein}, gene expression networks~\cite{yinggene}, fraud detection~\cite{Akoglu2015}, program optimization~\cite{demme2012,mcfarling89}, and epidemic spread~\cite{Newman2003}. 

As ground-truth communities are not known in real-world networks, clustering algorithms often model the quality of graph clustering using \emph{modularity} as an objective function. Modularity measures the density of links inside communities as compared to links between communities, while accounting for random chance~\cite{Girvan2002,Newman2006}. Intuitively, we optimize for modularity to identify densely connected regions of the network. Modularity is one of the most widely used objective functions for graph clustering and has been shown to be an effective objective function to optimize for~\cite{Newman2004}. However, modularity optimization is \textit{strongly} NP-complete~\cite{brandes2007modularity} and thus approximation and heuristic algorithms are used in practice. 

State-of-the-art clustering algorithms, such as \textsc{Louvain}~\cite{louvain} and \textsc{VieClus}~\cite{vieclus}, use heuristics to optimize for modularity in their objective function. These as well as other in-memory algorithms operate by loading the entire graph in memory. While these \emph{in-memory} approaches achieve high solution quality by leveraging complete global information to compute clusters, they come with significant memory overhead due to storing the entire edge set in memory. However, there is a growing need to process massive graphs with limited computational resources, particularly in real-time applications like online social network monitoring, fraud detection in financial networks, and dynamic routing in transportation systems~\cite{chhabra2024partitioningtrillionedgegraphs}. Reducing the memory requirements of graph clustering not only addresses these challenges but also significantly lowers monetary costs, enabling large-scale clustering on \hbox{small, cost-effective machines.}

\emph{Streaming} algorithms offer a scalable alternative to \emph{in-memory} graph clustering but often sacrifice solution quality. Instead of loading the entire edge set, these algorithms process nodes or edges sequentially, and assign nodes to clusters immediately. This approach significantly reduces memory overhead compared to \emph{in-memory} clustering algorithms but results in lower solution quality due to the lack of global graph knowledge. While streaming algorithms have been extensively studied for the related problem of graph partitioning~\cite{HeiStreamEdge,HeiStream,tsourakakis2014fennel}, their potential for graph clustering remains under-explored. Hollocou et al.~\cite{hollocou} propose a one-pass streaming clustering algorithm, but research on improving solution quality through techniques like \emph{re-streaming} or leveraging partial global information is limited. State-of-the-art re-streaming~\cite{nishimura2013restreaming} and buffered streaming~\cite{HeiStream,cuttana} graph partitioners demonstrate the potential of these methods to enhance solution quality in streaming scenarios. Therefore, there is considerable scope to develop a streaming clustering algorithm that reduces memory overhead while limiting the compromise in solution quality.

\paragraph*{Contributions} 
\begin{itemize}
	\item To address the need for a high-quality graph clustering algorithm with low memory overhead, we propose \algname, a graph \textbf{Clu}stering algorithm in a \textbf{St}reaming setting with multi-stage refinement using \textbf{R}e-streaming and \textbf{E}volutionary heuristics to incorporate partial global information. \algname~processes the graph in a node stream, dynamically constructing a quotient graph to refine clustering and optimize modularity using global information.
	\item We provide four configurations of \algname, which offer solution quality and resource consumption trade-offs. 
	\item Through experimental analysis, we demonstrate that our lightest mode achieves 89.8\% higher solution quality, runs 2.6$\times$ faster, and uses only 58.8\% of the memory required by the current state-of-the-art streaming graph clustering algorithm. Our strongest mode improves solution quality by up to 150\% on average over the state-of-the-art, setting a new benchmark. 
	\item We show the effect of \algname~'s modularity optimization in predicting ground truth communities of real world networks. \algname~improves the NMI score of ground-truth community recovery by 17\% on average over the state-of-the-art streaming graph clustering algorithm. 
\end{itemize}

\section{Preliminaries}
\label{sec:preliminaries}
\begin{figure*}
	\centering
	\includegraphics[width=0.3\textwidth]{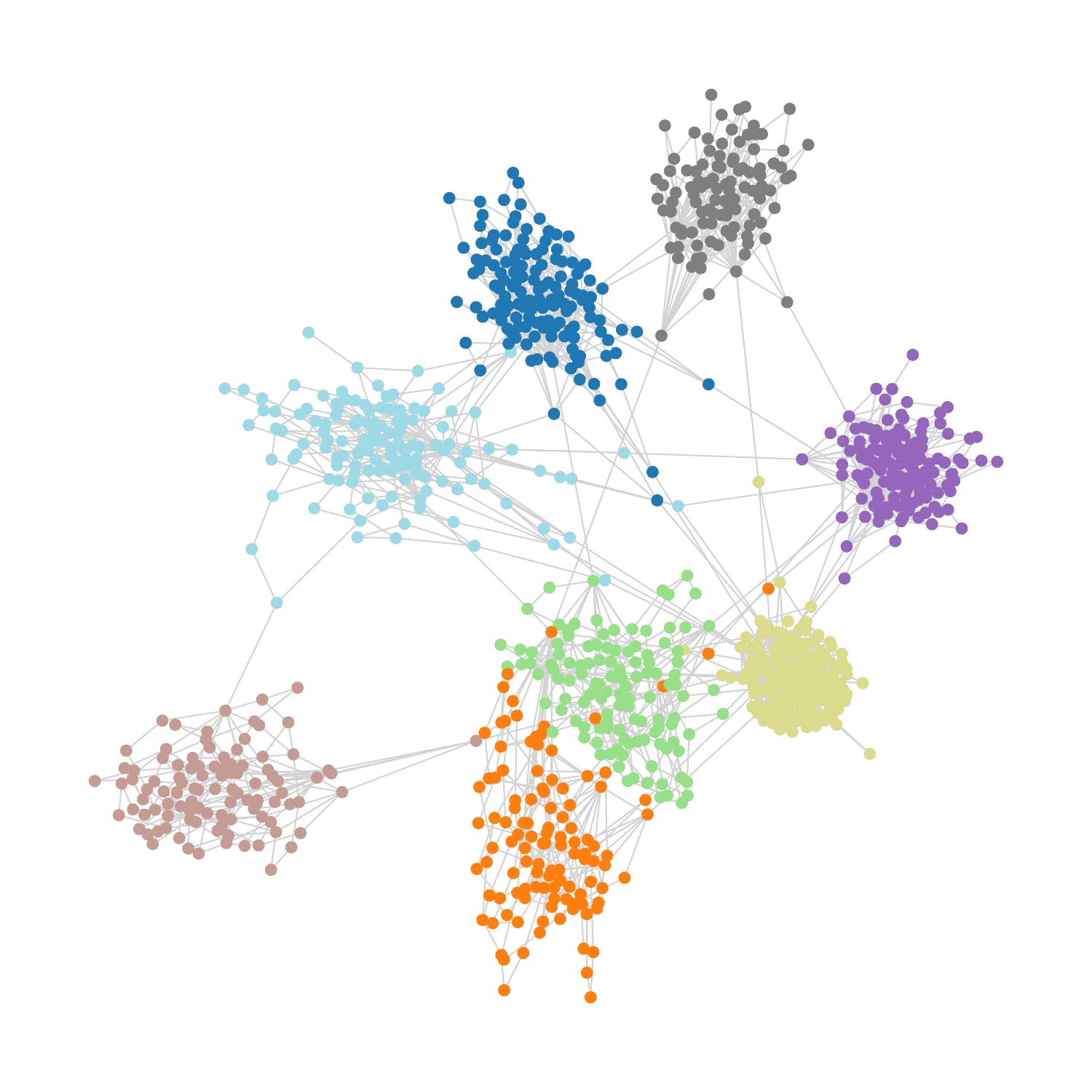} %
	\caption{A clustering (colored nodes) of a network into densely interconnected regions.}
	\label{fig:cluster_graph} 
\end{figure*}
\subsection{Basic Concepts}
\label{subsec:basic_concepts}

\subparagraph*{Graphs.}
Let $G=(V=\{0,\ldots, n-1\},E)$ be an \emph{undirected graph} with no multiple or self-edges, such that $n = |V|$, $m = |E|$.
Let $c: V \to \MdR_{\geq 0}$ be a node-weight function, and let $\omega: E \to \MdR_{>0}$ be an edge-weight function.
We generalize $c$ and $\omega$ functions to sets, such that $c(V') = \sum_{v\in V'}c(v)$ and $\omega(E') = \sum_{e\in E'}\omega(e)$.
An edge $e = (u, v)$ is said to be \emph{incident} on nodes $u$ and $v$. Let $N(v) = \setGilt{u}{(v,u) \in E}$ denote the neighbors of $v$. 
Let $d(v)$ be the degree of node $v$ and $d_w(v)$ be the weighted degree of a node $d_w(v) = \sum_{u \in N(v)} w(v, u)$.  
A graph $S=(V', E')$ is said to be a \emph{subgraph} of $G=(V, E)$ if $V' \subseteq V$ and $E' \subseteq E \cap (V' \times V')$.

\subparagraph*{Graph Clustering.}
The problem addressed in this paper is the \emph{graph clustering problem}. A clustering $\mathcal{C}$ is defined as a partition of the node set $V$ into a set of disjoint \textit{blocks/clusters} $C_1,\ldots,C_k$ such that $C_1 \cup \ldots \cup C_k = V$, as depicted in Figure~\ref{fig:cluster_graph}. Importantly, the number of clusters $k$ is typically not predetermined. A clustering is considered \textit{trivial} if it consists of only a single cluster containing all nodes or if all clusters are singletons (i.e., contain only one node). 
Each cluster $C_i$ can be identified with the subgraph of $G$ induced by its nodes. 
The set of \textbf{intra-cluster edges} is denoted as $E(C) := E \cap \bigcup_{i} (C_i \times C_i)$. These are edges with both endpoints in the same cluster $C_i$. The \textbf{inter-cluster edges} are given by $E \setminus E(C)$ which is the set of edges whose endpoints are in different clusters.
Various objective functions for graph clustering are proposed in the literature, which are discussed in the following paragraph. 

\subparagraph*{Objective Functions}
Several measures assess clustering quality, including \emph{coverage}~\cite{Gaertler2005}, \emph{performance}~\cite{Gaertler2005}, \emph{inter-cluster conductance}~\cite{kannan2004clusterings}, \emph{surprise}~\cite{arnau2004iterative}, \emph{map equation}~\cite{rosvall2009map}, and \emph{modularity}~\cite{newman2004finding}. We refer the reader to the given references for further insight into these objective functions.
Of these, \emph{coverage}, the fraction of intra-cluster edges, is a simple but flawed metric as it favors trivial solutions. \emph{Modularity} mitigates this by comparing observed coverage to its expected value under a random model.  In this work, we focus on the modularity objective function as it is a widely accepted clustering quality function~\cite{dimacschallengegraphpartandcluster,licommunitydetectionsurvey24}. \hbox{It is formally defined as:}
\begin{equation}
	\label{eq:modularity}
	Q(C) = \frac{1}{m} \sum_{C_i \in C} \left( K_{C_i \rightarrow C_i} - \frac{\text{vol}(C_i)^2}{2m} \right)
\end{equation}
where $K_{C_i \rightarrow C_i}$ is the total weight of intra-cluster edges in $C_i$, and $\text{vol}(C_i) = \sum_{v \in C_i} d_w(v)$ is the weighted volume of $C_i$. This formulation aligns with our modularity gain computation in Section~\ref{subsec:mod_opt}, where modularity optimization is performed iteratively using streaming updates.

\subparagraph*{Streaming Computation Model.}
Streaming algorithms typically follow a \emph{load-compute-store} logic. The most widely used streaming model is the one-pass model, which loads nodes along with their neighborhoods (in node streams) or edges (in edge streams) and immediately assigns nodes to clusters.
Initially, each node is assumed to be assigned to a singleton cluster. When the node arrives in the data stream, a scoring function computes the change in score when the node is moved from its current cluster to each of the clusters to which its neighbors are assigned. The scoring function indicates how well a node is connected to a cluster by examining cluster assignments of neighbors of the current node. The node is then assigned to the cluster that maximizes the \hbox{gain in score.} One configuration of our algorithm uses this one-pass streaming model to permanently assign nodes, however, our multi-stage refinement approach enables the re-assignment of nodes into clusters after the initial streaming, based on updates made to optimize the clustering. The current state-of-the-art streaming clustering algorithm by Hollocou et al.~\cite{hollocou} also reassigns clusters during its edge streaming process.

\subparagraph*{Memetic Algorithms.}
\emph{Evolutionary} or \emph{memetic} algorithms use population-based heuristics to mimic natural evolution to optimize solutions. We utilize the state-of-the-art memetic graph clustering algorithm \textsc{VieClus}~\cite{vieclus} within our approach. We chose \textsc{VieClus} because it outperforms the previous benchmark result of the 10th DIMACS graph clustering challenge in 98 out of 115 runs of the algorithm~\cite{vieclus}. Memetic algorithms like \textsc{VieClus} identify or generate possible solutions to a problem, and then recombine them to produce new and diverse solutions with higher quality, while also introducing random solutions to explore search space and escape local optima. This process is likened to natural evolution, where the initial solutions are termed \textit{individuals} within a \textit{population}, which are then recombined to form the improved solutions, termed \textit{offspring}, with random solutions generated through mutation. Recombination exploits characteristics of previous solutions to enable improvement in solution quality, and mutation fosters exploration by introducing random variability and preventing premature convergence. In memetic clustering, additional clustering solutions are generated through recombination and mutation to select the clustering that optimizes solution quality.

\subsection{Related Work}
\label{subsec:related_work}
There has been a significant amount of research on graph clustering; we refer the reader to~\cite{FORTUNATO201075,HamannMapEq,watteau2024advancedgraphclusteringmethods} for thorough reviews of contributions in this field. Here, we survey some graph clustering algorithms relevant to our contribution. 

Several successful algorithms have been developed for prominent applications of graph clustering. The \textsc{Louvain} method, introduced by Blondel et al.~\cite{louvain}, is a multi-level clustering algorithm that optimizes modularity as its objective function. The \textsc{Leiden} algorithm~\cite{Traag2019} provides adaptations to the \textsc{Louvain} method. \textsc{VieClus} also optimizes modularity as its objective function, but uses the heuristics of evolutionary algorithms to tackle the graph clustering problem~\cite{vieclus}.     
Spectral clustering~\cite{vonLuxburg2007} partitions graph data into clusters using the eigenvalues and eigenvectors of a similarity matrix and, along with $k$-means/$k$-median/$k$-center methods~\cite{JIANG2023691,wanclust2017}, is particularly effective when the desired number of clusters is predefined. Stochastic Block Models (SBM)~\cite{JIANG2023691,Lee2019} are probabilistic models that analyze graphs with underlying (latent) structures, where nodes are grouped into blocks, and edge density between blocks is determined by probability distribution. Markov Clustering (MCL)~\cite{dongen2008} efficiently clusters graphs by leveraging random walks and Markov chain properties to identify densely connected portions of a graph. \textsc{DBSCAN}~\cite{apachegiraph} and correlation clustering~\cite{Bansal2004} excel at detecting anomalously dense clusters. 
Additionally, several distributed graph clustering algorithms have been developed, including frameworks like \textsc{Pregel}/\textsc{Giraph}~\cite{pregel,apachegiraph} and \textsc{MapReduce}~\cite{mapreduce}. \textsc{TeraHAC}~\cite{terahac} is a distributed algorithm for hierarchical clustering that addresses the runtime bottlenecks of other hierarchical clustering algorithms. Some recent approaches to graph clustering draw on deep learning and graph neural networks~\cite{ijcai2020p693,liu2023surveydeepgraphclustering,su2024surveydeeplearningcommunity,wangdeepgraphnodeclust2024}.

While there are several successful approaches to the graph clustering problem, most state-of-the-art clustering algorithms are in-memory algorithms, with limited research undertaken on streaming graph clustering. Hollocou et al.~\cite{hollocou} introduce a streaming graph clustering algorithm that reads edge streams and assigns nodes to clusters on-the-fly. When streaming an edge $e=(u,v)$, they either (a) assign $u$ to the cluster of $v$, (b) assign $v$ to the cluster of $u$ or (c) leave both in their respective cluster, attempting to optimize for modularity. Although this algorithm requires very few computational resources compared to state-of-the-art in-memory algorithms, its solution quality is lower due to the absence of global knowledge of the graph and further impacted by sub-optimal cluster assignment decisions. Of related interest, Assadi et al.~\cite{assadi2022hierarchical} present a theoretical proof for the runtime and memory complexity of a streaming algorithm for the distinct hierarchical clustering problem.

\section{\algname: Streaming Graph Clustering with Memetic Refinement and Re-Streaming Local Search}
\label{sec:main_contribution}
In this section, we present our algorithm \algname. First, we provide an overview of \algname~'s light-weight streaming approach to graph clustering. Subsequently, we detail the steps of our approach, namely optimizing for modularity in a streaming setting (Section~\ref{subsec:mod_opt}) and the optional refinements: on-the-fly construction of a dynamic quotient graph model for memetic graph clustering (Section~\ref{subsec:memetic}), and re-streaming with local search (Section~\ref{subsec:restream}). 

\subsection{Overall Algorithm}
The \algname~algorithm addresses the quality limitations of one-pass streaming by incorporating global knowledge through memetic clustering of a quotient graph model and re-streaming with local search to optimize modularity. The algorithm processes the input graph $G$ node by node, loading the neighborhood $N(v)$ of only a single node $v$ into memory at a time. For each node $v$, we compute the cluster $C^*$, accounting for the clusters assigned to $N(v)$, that yields the highest modularity gain, and assign $v$ to that cluster.
Optionally, as each node is assigned, we dynamically construct a quotient graph $G_Q$, where each node represents a cluster in $G$, and edges model inter-cluster edges in $G$, supplemented by self-loops to represent intra-cluster edges. 
This quotient graph $G_Q$ is designed such that the modularity of any clustering of $G_Q$ is equivalent to the modularity of that clustering of the original graph $G$ (proven in Appendix Theorem~\ref{thm:modularity_invariance}). To further improve the quality of the clustering, a memetic graph clustering algorithm is run on the computed quotient graph $G_Q$. This allows us to split and merge clusters and drastically increases the search space to optimize modularity.
Once memetic clustering on the quotient graph is finished, the set of clusters with the highest modularity computed on $G_Q$ is used to update the clustering of the original graph. At this stage, the algorithm either outputs the resulting clusters or re-streams the graph to perform additional local search operations, iterating until modularity gain converges to the local optimum. The overall structure of \algname~is outlined in Algorithm~\ref{alg:overall}, with detailed explanations \hbox{provided in the subsequent sections.}

\begin{algorithm}[t]
	\caption{\algname: Overall Approach for Streaming Graph Clustering}
	\label{alg:overall}
	\begin{algorithmic}[1]
		\State \textbf{Input:} Graph $G = (V, E)$, \textcolor{blue}{Flags: \textbf{$refineG_Q$}}, \textcolor{darkgreen}{\textbf{$restreamLS$}}
		\State \textbf{Output:} Clustering of nodes $\mathcal{C}$
		
		\State Init array $\mathcal{C}[v] = v$ $\forall v \in V$ \Comment{\textcolor{gray}{singleton clustering $\Theta(n)$ memory}}
		\State Init empty quotient graph $G_Q$  \Comment{\textcolor{gray}{$\mathcal{O}(|E_{G_Q}|)$ memory}}
		
		\For{\textbf{each} node $v \in V$ (node stream)}
		\State $\mathcal{C}[v] = C^*$ $\gets$ \Call{computeCluster}{$v, \mathcal{C}$} \Comment{\textcolor{gray}{cluster maximizing modularity gain}}
		\If{\textcolor{blue}{\textbf{$refineG_Q$}}}
		\State \Call{updateQuotientGraph}{$v, \mathcal{C}, G_Q$} \Comment{\textcolor{gray}{on-the-fly $G_Q$ construction}}
		\EndIf
		\EndFor
		
		\If{\textcolor{blue}{\textbf{$refineG_Q$}}}
		\State \Call{refineClustering}{$\mathcal{C}, G_Q$} \Comment{\textcolor{gray}{memetic graph clustering refinement}}
		\EndIf
		
		\If{\textcolor{darkgreen}{\textbf{$restreamLS$}}} 
		\State \Call{restreamLocalSearch}{$\mathcal{C}$}
		\EndIf
		
		\State \Return $\mathcal{C}$
	\end{algorithmic}
\end{algorithm}

\subsection{One-Pass Streaming with Modularity Gain Scoring}
\label{subsec:mod_opt}
Our streaming algorithm processes nodes sequentially along with their neighborhoods, assigning them to clusters on-the-fly. At first, we assume that each node is part of its own singleton cluster. 
As nodes stream in, each node $v$ is assigned to the cluster $C^*$ that maximizes modularity gain, $\Delta Q_{v: C_{cur} \rightarrow C_{can}}$, computed using Equation~\ref{eq:delta_mod} (also used in the \textsc{Louvain} method~\cite{louvain}).
\begin{equation}
	\label{eq:best_cluster}
	C^* = \argmax_{C_{can} \in C(N(v))} \Delta Q_{v: C_{cur} \rightarrow C_{can}}
\end{equation}
The delta modularity function quantifies the change in modularity when node $v$ moves from its current cluster $C_{cur}$ to a candidate cluster $C_{can}$. For a given node $v$, we identify a set of candidate clusters $C(N(v))$ as the clusters to which neighbors of $v$ have been assigned. If all candidate clusters decrease modularity, we do not move the node and instead retain the singleton cluster for $v$. 
Formally, we compute delta modularity with Equation~\ref{eq:delta_mod}.
\begin{equation}
	\label{eq:delta_mod}
	\Delta Q_{v: C_{cur} \rightarrow C_{can}}
	= \frac{1}{m} (K_{v \rightarrow C_{can}} - K_{v \rightarrow C_{cur}})
	- \frac{d_w(v)}{2m^2} (d_w(v) + \text{vol}(C_{can}) - \text{vol}(C_{cur}))
\end{equation}
where we define $K_{v \rightarrow C}$ as the total weight of edges between node $v$ and its neighbors in cluster $C$ and $\text{vol}(C)$ to be the weighted volume of a cluster $\sum_{v \in C_i} d_w(v)$.
We compute the best cluster for each node in $O(\text{deg}(v))$ time, resulting in an overall streaming pass time complexity of $O(n\Delta)$, where $\Delta$ is the maximum degree of the graph. This is achieved by maintaining an array of cluster assignments for each node of size $\theta(n)$ and an array of cluster volumes of size $|C|$, i.e., total number of clusters. These structures are updated incrementally as nodes are processed. The modularity maximization approach is reminiscent of the first phase of the \textsc{Louvain} method for graph clustering; however, unlike the \textsc{Louvain} method, we visit each node only once and in the node stream order. 
\subsection{Modularity Refinement via Memetic Clustering}
\label{subsec:memetic}
\algname~optionally applies memetic graph clustering after streaming to enhance solution quality by integrating partial global information and evolutionary heuristics. This process maintains a dynamic quotient graph $G_Q$ during streaming, which then serves as the input for memetic clustering. This approach is similar to the \textsc{Louvain} method, which also contracts the initial clustering into a quotient graph; however, unlike the \textsc{Louvain} method we construct the quotient graph on the fly since we do not store all edges of the graph in-memory but stream them. Moreover, we only contract the clustering once before continuing with the memetic approach. Here, we apply the state-of-the-art memetic graph clustering algorithm used in \textsc{VieClus}~\cite{vieclus}. We detail the quotient graph construction and give a brief overview of the techniques used in the memetic clustering approach of \hbox{\textsc{VieClus} in order to be self-contained.}
\begin{figure*}[t]
	\centering
	\includegraphics[width=0.7\textwidth]{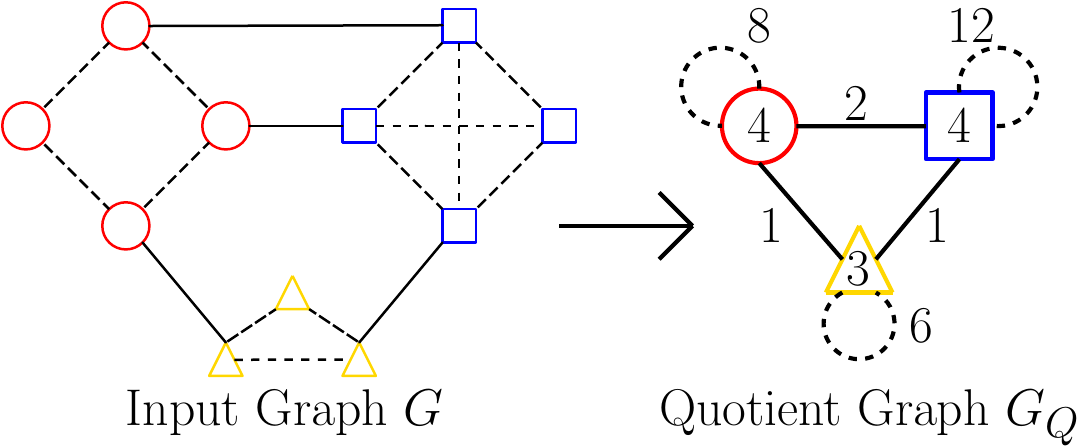} %
	\caption{A quotient graph $G_Q$ constructed from an undirected, unweighted toy graph $G$ (all edges have unitary weight). Here, clusters are represented by unique colors and shapes, thick lines show inter-cluster edges and dashed lines show intra-cluster edges. In the quotient graph, each cluster is contracted into a \textit{supernode} with weight equal to number of nodes in that cluster in $G$. Edges between the nodes of the quotient graph represent inter-cluster edges in $G$, with weight equal to the sum of the weight of the corresponding inter-cluster edges in $G$. Intra-cluster edges in $G$ are represented by weighted self-loops in the quotient graph, counted twice - once for each directed intra-cluster edge.}
	\label{fig:quotient_graph} 
\end{figure*}
In the quotient graph $G_Q$, each cluster $C_i$ of the input graph $G$ is represented as a \textit{supernode} $v'_i$, weighted by the number of assigned nodes in $C_i$. Edges in $G_Q$ are derived from the input graph: an edge between quotient nodes $v'_i$ and $v'_j$ corresponds to inter-cluster edges between $C_i$ and $C_j$ in $G$, weighted by the sum of inter-cluster edge weights between $C_i$ and $C_j$; self-loops from $v'_i$ to itself reflect intra-cluster connectivity, weighted by the sum of weights of intra-cluster edges in $C_i$. We count each undirected intra-cluster edge twice to set $w(v'_i,v'_i) = K_{C_i \rightarrow C_i}$ in the quotient graph. With this construction, the modularity computed on clusterings of the coarser quotient graph is equivalent to the modularity of corresponding clusterings of the input graph, as proven in Appendix Theorem~\ref{thm:modularity_invariance}.
Thus, the quotient graph provides a low-memory representation of the input graph and its clustering, allowing efficient access to partial global information to optimize modularity. 
Given modularity equivalence, by optimizing modularity of a clustering in $G_Q$, we obtain an updated clustering of $G$ with higher modularity. To construct $G_Q$ on the fly, we maintain a hash map of quotient graph edges $E_{G_Q}$, dynamically updated as nodes in $G$ are streamed and assigned to clusters (Algorithm~\ref{alg:update_quotient_graph}). This construction requires $O(E_{G_Q})$ additional memory and takes linear time in the size of $E_{G_Q}$. 
\begin{algorithm}[t]
	\caption{updateQuotientGraph$(v, Q, \mathcal{C})$: On-the-Fly Quotient Graph Construction}
	\label{alg:update_quotient_graph}
	\begin{algorithmic}[1]
		\State \textbf{Input:} Node $v$ (the node to process), Hashmap $Q$ representing quotient graph edges, where $Q[(C_i, C_j)]$ stores edge weights, Array $\mathcal{C}$ of cluster IDs for all nodes
		\State \textbf{Output:} Updated hashmap $Q$
		
		\State $C_i \gets \mathcal{C}[v]$ \Comment{\textcolor{gray}{cluster ID of current node $v$}}
		
		\For{\textbf{each} neighbor $u$ of $v$ with $v<u$} 
		
		\State $C_j \gets \mathcal{C}[u]$ \Comment{\textcolor{gray}{cluster ID of neighboring node $u$}}
		\State $w \gets w(u,v)$
		
		\If{$C_i == C_j$} 
		\State $w \gets 2 \cdot w$ \Comment{\textcolor{gray}{double the weight for self-loops}}
		\EndIf
		
		\If{$(C_i, C_j) \notin Q$ \textbf{or} $(C_j, C_i) \notin Q$}
		\State $Q[(C_i, C_j)] \gets w$ \Comment{\textcolor{gray}{insert new edge with key $(C_i, C_j)$ and value $w$}}
		\Else
		\State $Q[(C_i, C_j)] \gets Q[(C_i, C_j)] + w$ \Comment{\textcolor{gray}{increment edge weight (value)}}
		\EndIf
		
		\EndFor
		
		\State \textbf{return} $Q$
	\end{algorithmic}
\end{algorithm}

\begin{algorithm}[ht]
	\caption{RestreamLocalSearch($\mathcal{C}, X\in[0,1], T_{\text{cutoff}}$)}
	\label{alg:restream}
	\begin{algorithmic}[1]
		\State \textbf{Init:} $ActiveNodes \gets V$, $\Delta Q \gets \infty$, $startTime \gets \Call{currTime}{}$
		\While{$ActiveNodes \neq \emptyset$ \textbf{and} $\Delta Q \geq X \cdot Q_{\text{total}}$ \textbf{and} $\Call{currTime}{} - startTime < T_{\text{cutoff}}$}
		\State $NextActiveNodes \gets \emptyset$, $\Delta Q \gets 0$
		\For{each $v \in ActiveNodes$ \textbf{(node stream)}}
		\State $C^* \gets$ \Call{computeCluster}{$v, \mathcal{C}$}
		\If{$C^* \neq \mathcal{C}(v)$}
		\State Update $\mathcal{C}[v] \gets C^*$
		\State Add neighbors of $v$ to $NextActiveNodes$
		\State $\Delta Q \gets \Delta Q + \Delta Q_v$ \Comment{\textcolor{gray}{$\Delta Q_v$ = modularity gain from moving node $v$}}
		\EndIf
		\EndFor
		\State $ActiveNodes \gets NextActiveNodes$
		\State $Q_{\text{total}} \gets Q_{\text{total}} + \Delta Q$
		\EndWhile
	\end{algorithmic}
\end{algorithm}

We apply the state-of-the-art memetic graph clustering algorithm used in \textsc{VieClus}~\cite{vieclus} on the quotient graph $G_Q$ to iteratively refine clustering through evolutionary heuristics. The process in \textsc{VieClus} begins by generating a diverse population of clustering solutions for $G_Q$ using an adapted \textsc{Louvain} method with label propagation~\cite{sizeconstrainedlabelprop2014}. This initial solution space is used to generate new solutions through iterative rounds of recombination and mutation.

Each round of recombination starts by selecting two sets of clusters with high modularity using a tournament selection rule~\cite{millergenetictournament} from the initialized solution space to serve as input sets to generate a new clustering solution. The input or parent sets are combined using overlay clustering: two nodes from $G_Q$ belong to the same cluster in the overlay clustering if and only if they are clustered together in both input sets. Intuitively, overlay clustering enforces a stricter agreement between input sets; if both inputs independently classify two nodes as belonging to the same cluster, then there is high confidence that they should be grouped together.
\textsc{VieClus} employs two recombination strategies: flat and multilevel recombination. The former contracts the overlay clusters into a new quotient graph and refines them using \textsc{Louvain}, and the latter uses multi-level local search to refine the solution produced by overlay clustering. Both methods use strategies to ensure the quality of the recombined solution is at least as good as that of its input sets~\cite{vieclus}. 

Mutation is incorporated to counteract the reduction in solution diversity caused by recombination, ensuring a broader search space and avoiding local optima. To achieve this, two input sets are selected from the initialized population, and their clusters are split into two balanced blocks using graph partitioning. The modified solutions, now with additional clusters, serve as inputs for multi-level recombination.

After computing the offspring set, we evaluate it against the existing clustering solutions in the initial population. If it has higher or equal modularity than any existing solution, it replaces the most similar of those solutions, i.e., the set with the smallest symmetric difference between their sets of inter-cluster edges, thus maintaining diversity and preventing premature convergence. Otherwise, the offspring set is discarded. The process then repeats, selecting new parent solutions for recombination or mutation. This iterative process continues until a time cutoff is reached, after which the highest-modularity clustering in the population is applied to update the clustering in $G$.

\subsection{Modularity Refinement via Re-Streaming with Local Search}
\label{subsec:restream}
Our algorithm allows for further modularity optimization via re-streaming with local search. The re-streaming step is applicable after an initial clustering is obtained from the first round of streaming, irrespective of whether we build and optimize the quotient graph. In our re-streaming approach, we first process the node stream again, and identify modularity gain movements for each node by using Equation~\ref{eq:delta_mod}. In this re-stream, the change in modularity computation is more informed since higher-quality clusters have already been identified for the nodes in the first pass (instead of assuming they are all in singleton clusters). Here, unlike in the first round of streaming, we don't create new clusters, but merge and split the existing ones for modularity gain. 

While the first re-stream requires processing the entire node stream, subsequent re-streams employ an optimization that significantly reduces I/O time. Specifically, after the first re-stream, we track the neighborhoods of nodes that changed clusters during the process -- these neighborhoods define the \textit{active nodes} for the next iteration. In subsequent re-streams, instead of reading the entire node stream, we only load the \textit{active nodes} from disk, significantly reducing I/O overhead. The set of \textit{active nodes} is updated dynamically in each re-stream, consisting of the neighborhoods of nodes that were reassigned in the previous iteration. This process continues until (a) convergence, i.e., when no further node movements improve modularity and no active nodes remain, or (b) improvement in the round falls below $X\%$ of the modularity score computed so far, where $X$ is an input parameter, or (c) a specified time cut-off for the local search phase is exceeded. A pseudocode of the re-stream algorithm is shown in Algorithm~\ref{alg:restream}.

To prioritize reducing I/O time, we restrict cluster reassignments to those that alter the first component of the modularity gain function (Equation~\ref{eq:delta_mod}). The key intuition is that only the reassignment of nodes whose neighborhoods changed in the previous re-stream can impact this component. When a node’s neighbor is reassigned, it alters the total edge weight between the node and its neighbors in its current cluster, relative to the edge weight between the node and its neighbors in the candidate cluster -- this difference is captured in the first component of the modularity gain function. Conversely, if a node’s neighborhood remains unchanged, reassigning the node has no effect on this component, as the relative total edge weights remain the same. By focusing re-streaming on nodes with changed neighborhoods, we strike a balance between optimizing I/O efficiency and preserving significant updates to the modularity gain score.

\vfill
\section{Experimental Evaluation}
\label{sec:Experimental Evaluation}
Our experimental analysis investigates the following research questions:

\begin{itemize}
	\item \textbf{RQ1:} What is the impact of memetic clustering and re-streaming on the clustering quality, runtime, and memory consumption of \algname?
	
	\item \textbf{RQ2:} How does \algname's solution quality compare to the state-of-the-art streaming clustering algorithm and in-memory clustering methods? Specifically, how much closer does \algname~get to in-memory clustering than existing streaming algorithms?
	
	\item \textbf{RQ3:} How does \algname~balance runtime and memory consumption with solution quality, compared to other streaming and in-memory clustering approaches?
	
	\item \textbf{RQ4:} How well does \algname~recover ground-truth communities compared to existing streaming algorithms?
\end{itemize}

\subparagraph*{Datasets.} 
Our graph instances, listed in Table~\ref{tab:modularity_scores}, are sourced from various benchmark datasets~\cite{benchmarksfornetworksanalysis,BMSB,BRSLLP,BoVWFI,KaGen,snap,nr-aaai15}. To ensure diversity in size and domain, we include road networks and social networks from the SNAP dataset~\cite{snap}, and web graphs from the 10th DIMACS Graph Clustering Implementation Challenge~\cite{benchmarksfornetworksanalysis}. Additional graphs are sourced from the Network Repository~\cite{nr-aaai15} and the Laboratory for Web Algorithmics~\cite{BMSB,BRSLLP,BoVWFI}.
We also include synthetic graphs, such as \texttt{RGG} (Random Geometric Graphs) and \texttt{RHG} (Random Hyperbolic Graphs), which model real-world scale-free networks~\cite{KaGen}.
The ground truth communities in the SNAP dataset cannot be used reliably for evaluation in our study as their communities overlap, i.e., one node can be in multiple ground truth clusters; thus we do not use these datasets to evaluate ground truth recovery.
Instead, to answer \textbf{RQ4}, we incorporate citation and co-purchase networks with ground truth communities from PyTorch Geometric: \instance{CORA}, \instance{CiteSeer}, and \instance{PubMed}~\cite{cora} are citation networks where nodes represent academic papers and edges denote citations, with communities corresponding to research topics; \instance{AmazonCoPurchase}\cite{shchur2019pitfallsgraphneuralnetwork} represents product co-purchase data, where nodes are products and edges indicate frequently bought-together relationships, with communities corresponding to product categories.
For our experiments, all graphs were converted to the METIS node-stream format, with parallel edges, self-loops, and directions removed~\cite{karypis1998fast}. We assigned unit weights to all nodes and edges for consistency in evaluation.

\subparagraph*{Baselines.}
We compare \algname~against the state-of-the-art streaming graph clustering algorithm by Hollocou et al.~\cite{hollocou}, referred to as \textsc{Hollocou}. 
We obtained its C++ implementation from the official GitHub repository but found that it does not support streaming from disk; instead, it first loads the entire edge set into memory before processing edges. To ensure a fair comparison, particularly in terms of memory consumption, we modified their implementation to stream edges directly from disk.
\textsc{Hollocou} requires a parameter, $v_{max}$, which limits the maximum volume of any cluster. The authors provide no guidance on selecting this value or the one used in their experiments. We tested multiple settings, including $v_{max} = n$, and found $v_{max} = 10,000$ to yield the best clustering quality in our setup. Thus, all reported experiments are with $v_{max} = 10,000$.
Additionally, we compare \algname~with the in-memory algorithms \textsc{Louvain}, and \textsc{VieClus}, both obtained from \textsc{VieClus}'s official GitHub repository. We allocate \textsc{VieClus} \hbox{five minutes for evolutionary rounds.}

\subparagraph*{Experimental Setup and Reproducibility.} 
All experiments were performed on a single core of a machine with a sixteen-core Intel Xeon Silver 4216 processor running at $2.1$ GHz, $100$ GB of main memory, $16$ MB of L2-Cache, and $22$ MB of L3-Cache running Ubuntu 20.04.1. %
We implemented \algname~in C++ and compiled it using gcc 13.2.0 with full optimization enabled (-O3 flag). In the study presented here, we configure \algname~by setting a time limit of 15 seconds for the evolutionary rounds in memetic clustering. Additionally, for local search, we set a time limit of 10 minutes and set $X = 0.05$, that is, we stop local search when improvement falls below 5\%. The code for \algname~will be publicly available on acceptance of the paper.

\subparagraph*{Methodology}
We measure running time, solution quality, i.e., modularity scores, and memory consumption, i.e., the maximum resident set size for the executed process.
When averaging over all instances, we use the geometric mean to give every instance the same influence on the final score.
Let the runtime, modularity score, or memory consumption be denoted by the score $\sigma_{A}$ for some clustering generated by an algorithm $A$.
We express this score relative to others using the following tools:
\emph{improvement} over an algorithm $B$, computed as a percentage $(\frac{\sigma_A}{\sigma_B} - 1) * 100 \%$ and
\emph{relative} value over an algorithm $B$, computed as $\frac{\sigma_A}{\sigma_B}$.
Additionally, we present performance profiles by Dolan and Mor{\'e}~\cite{pp} to benchmark our algorithms.
These profiles relate the running time (resp. solution quality, memory) of the slower (resp. worse) algorithms to the fastest (resp. best) one on a per-instance basis.
Their $x$-axis shows a factor $\tau$ while their $y$-axis shows the fraction of instances for which an algorithm has up to $\tau$ times the running time (resp. solution quality, memory) of the fastest \hbox{(resp. best) algorithm}.
For graphs without ground truth communities, we use modularity to measure quality, and on graphs with ground truth communities, we use Normalized Mutual Information (NMI)~\cite{Lancichinetti_2009} to measure similarity with predicted clusters. Both metrics are widely used and regarded as reliable metrics for clustering quality~\cite{dimacschallengegraphpartandcluster,licommunitydetectionsurvey24}. 
	
\subsection{Impact of Re-Streaming and Memetic Clustering}
To answer \textbf{RQ1}, we evaluate the effect of our multi-stage refinement by analyzing four configurations of \algname, which users of our code can toggle between:
\begin{tcolorbox}
	\begin{description}
		\item[\algname-Light:] One-pass Streaming.
		\item[\algname-Light+:] Stream + Re-Streaming Local Search (LS).
		\item[\algname-Evo:] Stream + Memetic Quotient Graph Clustering.
		\item[\algname-Strong:] Stream \hbox{+ Memetic Quotient Graph Clustering + Re-Streaming LS.}
	\end{description}
\end{tcolorbox}
\begin{figure}[t]
	\centering
		\includegraphics[width=1.0\textwidth]{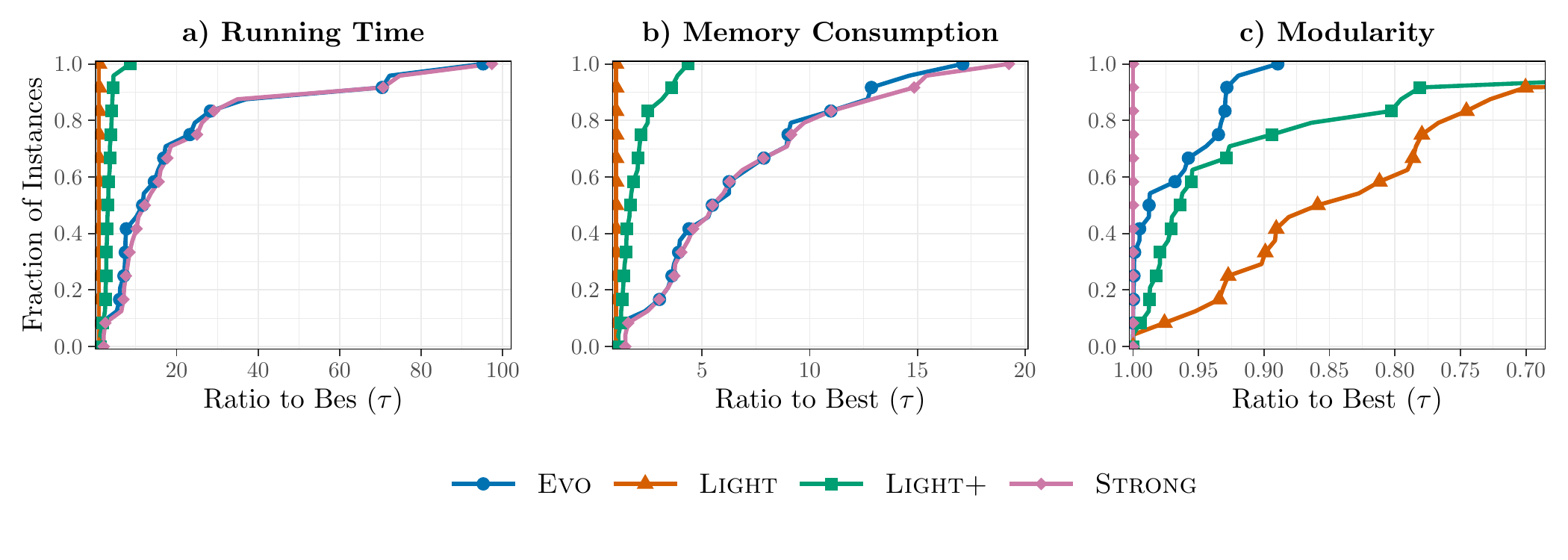}
		\caption{\textbf{Comparison of \algname~modes using performance profiles.} Each plot shows the fraction of instances (y-axis) for which an algorithm achieves a given ratio ($\tau$) to the best-performing method (x-axis). Note: in plot (c), the x-axis is decreasing since higher modularity is better. 
		}
		\label{plot:pp}
	\end{figure}
Experimental comparison of our four modes highlights their distinct applications based on computational constraints and solution quality priorities. Their trade-offs in solution quality and resource usage are illustrated in the performance profiles in Figure~\ref{plot:pp}.
\algname-Light is the fastest and most lightweight variant across all instances but achieves the lowest solution quality, as shown in Figure~\ref{plot:pp}, making it suitable for efficiency-prioritized applications.
We observe that re-streaming significantly improves clustering quality, at the cost of minimal runtime and memory. On average across all instances, \algname-Light+ increases solution quality by 15.7\% over \algname-Light while being 3.3$\times$ slower and using 1.9$\times$ more memory. Memetic clustering also enhances solution quality, with \algname-Evo achieving, on average across all instances, a 27.3\% improvement over \algname-Light, though at an increased computational cost, being 12.5$\times$ slower and using 5.5$\times$ more memory. 
\algname-Strong, which integrates both refinement strategies, delivers the highest solution quality on all instances (Figure~\ref{plot:pp}(c)), offering 31.5\% improvement over \algname-Light, while taking 14.2$\times$ more time and using 5.8$\times$ more memory on average across all instances. 
In absolute terms, \algname-Strong still has reasonable memory consumption; on our largest test instance, \texttt{uk-2007-05}, which has 105 million nodes and 3.3 billion edges, \algname-Strong requires only 2.18GB peak memory.   
Although our experimental results show that memetic clustering incurs a high running time cost, its duration is tunable and can be adjusted to reduce running time. \algname-Strong is thus ideal for applications where quality takes precedence over efficiency. \algname-Light+ provides a balanced alternative: as shown in Figure~\ref{plot:pp}, it achieves solution quality within about 10\% of \algname-Strong in approximately 70\% of instances while maintaining a significantly lower computational cost. 
\begin{tcolorbox}
	\textbf{Observation 1.} \algname-Light is the most efficient in runtime and memory but has the lowest quality.
	\algname-Light+ balances performance, providing 15.7\% quality improvement over \algname-Light while being 3.3$\times$ slower and using 1.9$\times$ more memory on average.
	\algname-Strong achieves the highest quality, improving by 31.5\% over \algname-Light on average, at the highest computational cost.
\end{tcolorbox}
\subsection{Solution Quality Compared to Existing Approaches}
In Table~\ref{tab:modularity_scores}, we present our experimental results for benchmarking \algname~against \textsc{Hollocou}, \textsc{louvain} and \textsc{VieClus}. Our results confirm that all \algname~modes outperform \textsc{Hollocou}, demonstrating its superior solution quality: on average across all instances, our lightest mode \algname-Light produces 89.8\% better solution quality than \textsc{Hollocou}, while \algname-Strong achieves a 149.5\% improvement over \textsc{Hollocou}, setting a new benchmark for streaming clustering. Notably, on the \texttt{com-friendster} network, \algname-Strong achieves a modularity score more than 16$\times$ better than that of \textsc{Hollocou}. 
\hbox{\textsc{Hollocou} underperforms in our study because its clustering decisions are} highly sensitive to the order in which edges are streamed -- solution quality improves when intra-cluster edges are streamed early, a condition that cannot be guaranteed. In the official implementation, they load the entire edge set into memory, randomize the order of edges to try to get intra-cluster edges to arrive earlier, and then process them sequentially.
We modified \textsc{Hollocou} to read/stream edges from the disk in the order in which they appear in the input data set (to ensure consistency across all algorithms) and process them in a streaming setting. 
Note, however, that in our experiments, the unmodified implementation -- which stores all edges in memory -- does not improve solution quality. However, it consumes, on average, 7$\times$ more memory than our modified streaming version. For example, on \texttt{com-friendster}, the official implementation required 27GB of memory while our streaming one used 0.98GB with comparable solution quality.  
Additionally, \textsc{Hollocou}'s reliance on the critical parameter \(v_{\max}\) poses a tuning challenge, as there is no principled method for selecting an optimal value for a given graph.

Compared to in-memory algorithms \textsc{louvain} and \textsc{VieClus}, \algname-Strong achieves 96.8\% of \textsc{Louvain}'s solution quality and 96.5\% of \textsc{VieClus}’ solution quality on average, demonstrating its capability to leverage partial global information through the quotient graph model, and the effectiveness of re-streaming with local search.

\begin{tcolorbox}
	\textbf{Observation 2.} \algname~surpasses \textsc{Hollocou} in solution quality in all configurations, offering improvements between 89.8\%(\algname-Light) to 149.5\%(\algname-Strong) on average. To answer \textbf{RQ2}, \algname~achieves new state-of-the-art \hbox{modularity} optimization in streaming graph clustering while achieving over 96\% of the solution quality of in-memory clustering algorithms \textsc{Louvain} and \textsc{VieClus}. 
\end{tcolorbox}
\begin{table}[t]
	\centering
	\caption{\textbf{Modularity Comparison.} We compare modularity scores achieved by our proposed algorithms against competing approaches across various graph instances. The table includes graph type and size (nodes, edges) to contextualize performance variations. \textsc{Hollocou} is a direct competitor streaming algorithm, while \textsc{Louvain} and \textsc{VieClus} are in-memory clustering algorithms. Missing instances (``-'') indicate failed runs due to exceeding the maximum amount of memory on the machine. Results demonstrate that all our proposed configurations outperform \textsc{Hollocou}.}
	\label{tab:modularity_scores}
	\resizebox{\textwidth}{!}{%
		\begin{tabular}{|lll|cccc|c|c|c|}
			\toprule
			\textsc{Graph} & \textsc{Type} & \textsc{Size: n,m}
			& \textsc{Light} & \textsc{Light+} & \textsc{Evo} & \textsc{Strong}
			& \textsc{Hollocou} & \textsc{Louvain} & \textsc{VieClus} \\
			\midrule
			\rowcolor{shade1}
			circuit5m & Circuit & 5.56M, 26.98M
			& 0.1926 & 0.2920 & 0.8208 & \textbf{\textcolor{PineGreen}{\underline{0.8214}}}
			& 0.3707 & 0.8163 & 0.8211 \\
			
			\rowcolor{shade2}
			cit-Patents & Citation & 3.77M, 16.52M
			& 0.1839 & 0.4704 & 0.8054 & \textbf{\textcolor{PineGreen}{\underline{0.8320}}}
			& 0.1828 & 0.8330 & 0.8334 \\
			
			\rowcolor{shade1}
			coAuthorsCitseer & Citation & 434K, 16.04M
			& 0.7991 & 0.8208 & 0.8751 & \textbf{\textcolor{PineGreen}{\underline{0.8860}}}
			& 0.6694 & 0.9019 & 0.9060 \\
			
			\rowcolor{shade2}
			com-amazon & Social & 334K, 925K
			& 0.7637 & 0.7971 & 0.9109 & \textbf{\textcolor{PineGreen}{\underline{0.9227}}}
			& 0.6420 & 0.9315 & 0.9345 \\
			
			\rowcolor{shade1}
			com-dblp & Social & 317K, 1.04M
			& 0.7249 & 0.7600 & 0.7455 & \textbf{\textcolor{PineGreen}{\underline{0.7760}}}
			& 0.5488 & 0.8282 & 0.8387 \\
			
			\rowcolor{shade2}
			com-friendster & Social & 65.6M, 1.81B
			& 0.5229 & 0.5839 & 0.5459 & \textbf{\textcolor{PineGreen}{\underline{0.5871}}}
			& 0.0334 & - & - \\
			
			\rowcolor{shade1}
			com-lj & Social & 3.99M, 34.68M
			& 0.6364 & 0.6933 & 0.6572 & \textbf{\textcolor{PineGreen}{\underline{0.7079}}}
			& 0.2056 & 0.7639 & 0.7676 \\
			
			\rowcolor{shade2}
			com-youtube & Social & 1.13M, 2.98M
			& 0.6426 & 0.6815 & 0.6610 & \textbf{\textcolor{PineGreen}{\underline{0.6902}}}
			& 0.2427 & 0.7233 & 0.7331 \\
			
			\rowcolor{shade1}
			com-orkut & Social & 3.07M, 117M
			& 0.5572 & 0.6248 & 0.5971 & \textbf{\textcolor{PineGreen}{\underline{0.6324}}}
			& 0.0792 & 0.6698 & 0.6737 \\
			
			\rowcolor{shade2}
			eu-2005 & Web & 862K, 16.13M
			& 0.7364 & 0.8895 & 0.9204 & \textbf{\textcolor{PineGreen}{\underline{0.9316}}}
			& 0.2683 & 0.9398 & 0.9414 \\
			
			\rowcolor{shade1}
			fullchip & Circuit & 2.98M, 11.81M
			& 0.4020 & 0.4482 & 0.5279 & \textbf{\textcolor{PineGreen}{\underline{0.5740}}}
			& 0.3569 & 0.5967 & 0.5994 \\
			
			\rowcolor{shade2}
			it-2004 & Web & 41.29M, 1.02B
			& 0.7433 & 0.9405 & 0.9646 & \textbf{\textcolor{PineGreen}{\underline{0.9693}}}
			& 0.2387 & 0.9762 & 0.9762 \\
			
			\rowcolor{shade1}
			italy-osm & Road & 6.68M, 7.01M
			& 0.9500 & 0.9601 & \textbf{\textcolor{PineGreen}{\underline{0.9976}}} & \textbf{\textcolor{PineGreen}{\underline{0.9976}}}
			& 0.9553 & 0.9980 & 0.9980 \\
			
			\rowcolor{shade2}
			rgg\_n26 & Rand. Geo. & 67.1M, 574.55M
			& 0.9671 & 0.9733 & 0.9897 & \textbf{\textcolor{PineGreen}{\underline{0.9909}}}
			& 0.8044 & 0.9956 & 0.9956 \\
			
			\rowcolor{shade1}
			rhg2b & Rand. Hyp. & 100M, 1.99B
			& 0.9920 & 0.9921 & \textbf{\textcolor{PineGreen}{\underline{0.9921}}} & \textbf{\textcolor{PineGreen}{\underline{0.9921}}}
			& 0.6757 & - & - \\
			
			\rowcolor{shade2}
			great-britain & Road & 7.73M, 8.15M
			& 0.9248 & 0.9264 & 0.9972 & \textbf{\textcolor{PineGreen}{\underline{0.9973}}}
			& 0.9344 & 0.9976 & 0.9977 \\
			
			\rowcolor{shade1}
			roadnet-ca & Road & 1.96M, 2.76M
			& 0.7802 & 0.7967 & 0.9919 & \textbf{\textcolor{PineGreen}{\underline{0.9923}}}
			& 0.8376 & 0.9928 & 0.9935 \\
			
			\rowcolor{shade2}
			roadnet-pa & Road & 1.08M, 1.54M
			& 0.7712 & 0.7869 & 0.9889 & \textbf{\textcolor{PineGreen}{\underline{0.9894}}}
			& 0.8333 & 0.9901 & 0.9910 \\
			
			\rowcolor{shade1}
			sk-2005 & Web & 50.63M, 1.81B
			& 0.7241 & 0.9433 & 0.9646 & \textbf{\textcolor{PineGreen}{\underline{0.9714}}}
			& 0.1450 & - & - \\
			
			\rowcolor{shade2}
			soc-flixster & Social & 2.52M, 7.92M
			& 0.4944 & 0.5145 & 0.5368 & \textbf{\textcolor{PineGreen}{\underline{0.5755}}}
			& 0.1539 & 0.6111 & 0.6191 \\
			
			\rowcolor{shade1}
			soc-lj & Social & 4.84M, 42.85M
			& 0.6416 & 0.6937 & 0.6689 & \textbf{\textcolor{PineGreen}{\underline{0.7196}}}
			& 0.2074 & 0.7626 & 0.7660 \\
			
			\rowcolor{shade2}
			in-2004 & Web & 1.38M, 13.59M
			& 0.7118 & 0.9351 & 0.9780 & \textbf{\textcolor{PineGreen}{\underline{0.9787}}}
			& 0.4468 & 0.9803 & 0.9805 \\
			
			\rowcolor{shade1}
			uk-2007-05 & Web & 105.89M, 3.30B
			& 0.7135 & 0.8675 & 0.8216 & \textbf{\textcolor{PineGreen}{\underline{0.8789}}}
			& 0.1838 & - & - \\
			
			\rowcolor{shade2}
			webbase-2001 & Web & 118.14M, 854.80M
			& 0.6341 & 0.7872 & 0.7194 & \textbf{\textcolor{PineGreen}{\underline{0.8088}}}
			& 0.5217 & 0.9822 & 0.9824 \\
			\bottomrule
			\textsc{GeoMean} & - & -
			& 0.6302 & 0.7292 & 0.8019 & \textbf{\textcolor{PineGreen}{\underline{0.8285}}}
			& 0.3321 & N/A & N/A \\
			\bottomrule
		\end{tabular}
	}
\end{table}

\subsection{Resource Consumption Compared to Existing Approaches}
\label{subsec:rq3}	
Beyond solution quality, we evaluate \algname’s computational efficiency to answer \textbf{RQ3}. In Figure~\ref{plot:boxplotstream} and Figure~\ref{plot:boxplot}, we present box plots showing the distribution of running time and memory consumption for the various algorithms. Figure~\ref{plot:boxplotstream} benchmarks the streaming algorithms across all graph instances; Figure~\ref{plot:boxplot} depicts all baseline algorithms, excluding, for all algorithms, the graph instances for which \textsc{VieClus} and \textsc{Louvain} failed due to their inability to process larger networks without exceeding memory available on the machine. We observe that \algname-Light is both faster and uses less memory than \textsc{Hollocou} while delivering higher solution quality (Table~\ref{tab:modularity_scores}), as seen in Figure~\ref{plot:boxplotstream}, which shows that the median running time and memory consumption across all instances for \algname-Light is lower than that of \textsc{Hollocou}. As a result, our algorithm outperforms the state-of-the-art streaming approach across all key metrics of efficiency and quality. On average across all instances, \textsc{Hollocou} is 2.6$\times$ slower and consumes 1.7$\times$ more memory than \algname-Light. \algname-Light+ strikes a strong balance, improving solution quality over \textsc{Hollocou} by 119.6\% (shown in Table~\ref{tab:modularity_scores}), with comparable memory consumption (1.1$\times$ more) and running times (1.3$\times$ slower), on average across all instances.
\begin{figure}[t]
	\centering
                \vspace*{-.5cm}
		\includegraphics[width=1.0\textwidth]{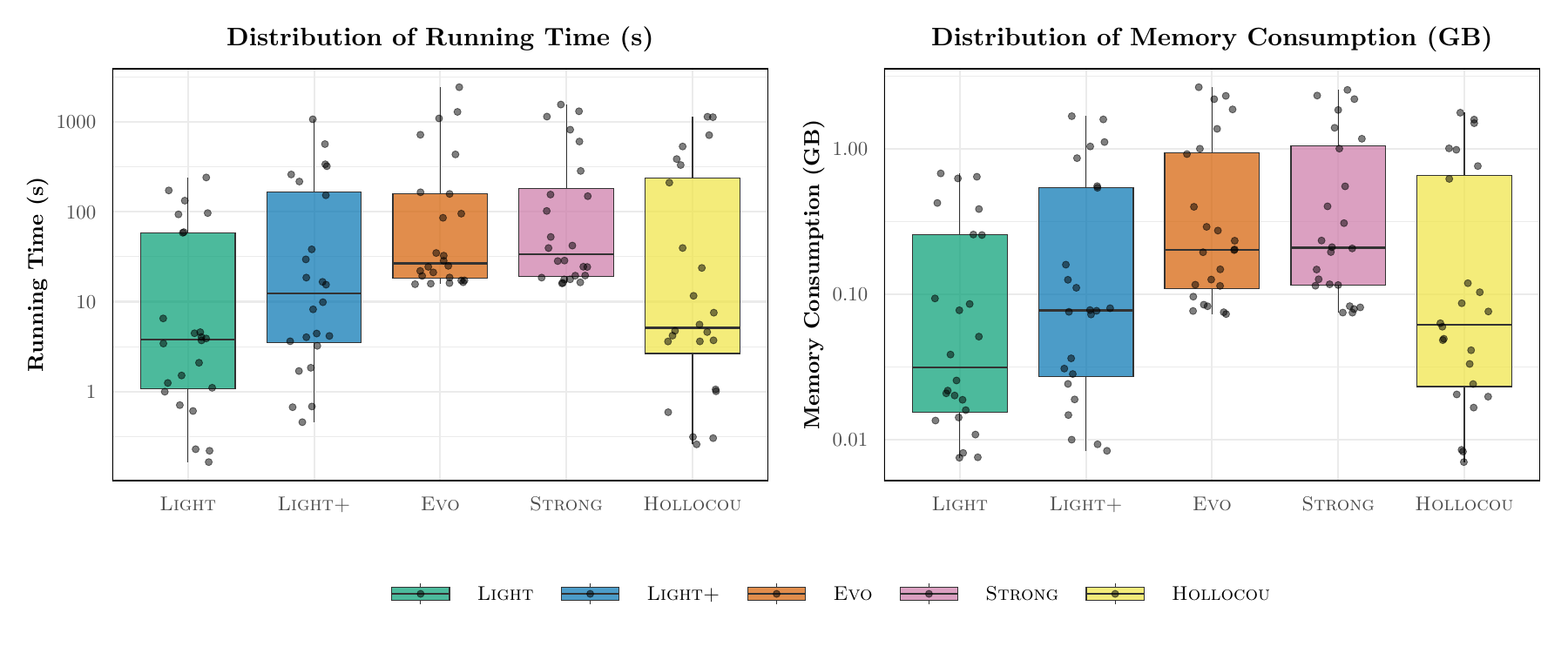}
		\caption{Comparison of \algname~modes with \textsc{Hollocou} in terms of running time and memory consumption. Box plots depict the distribution of running time (left) and memory consumption (right) across all instances. Note the logarithmic scale on the y-axis. All test instances are included.}
		\label{plot:boxplotstream}
                \vspace*{-.25cm}
	\end{figure}
\begin{figure}[h!]
	\centering
		\includegraphics[width=1.0\textwidth]{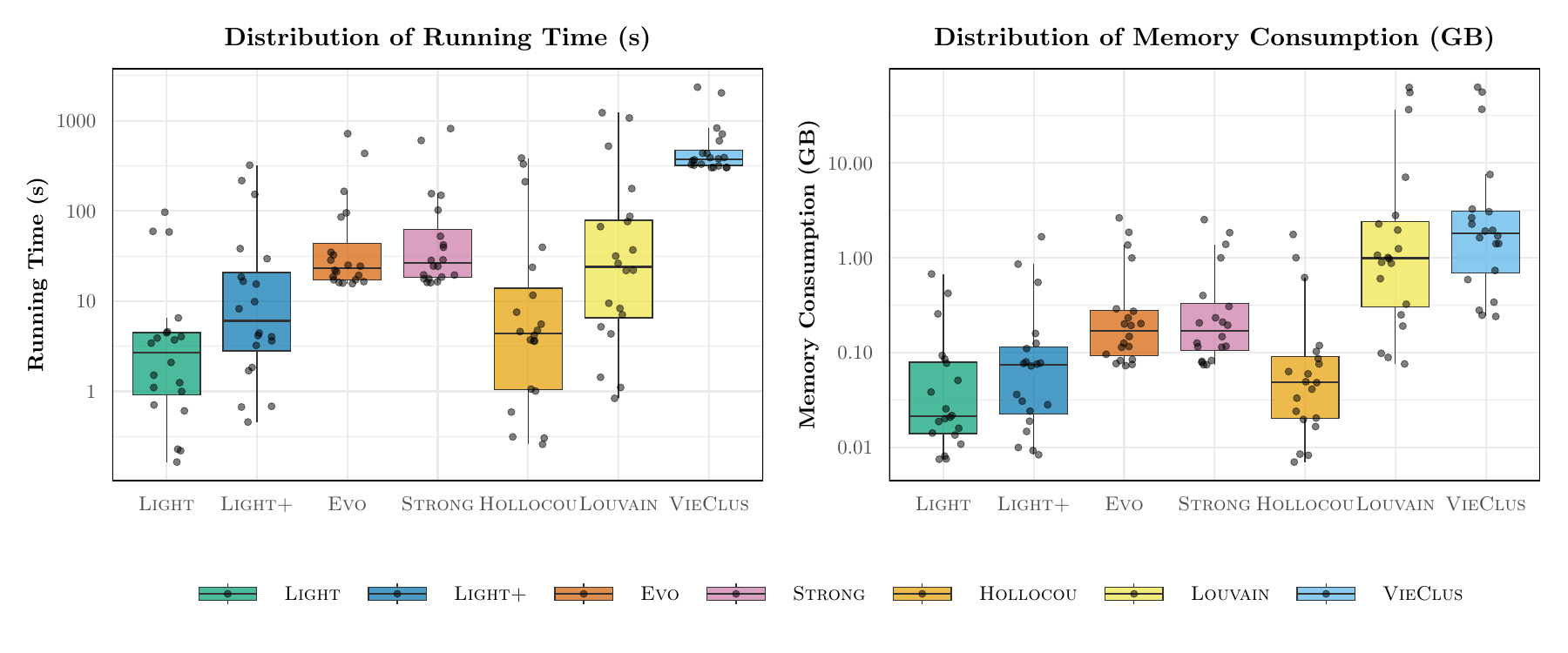}
		\caption{Comparison of \algname~modes with \textsc{Hollocou}, \textsc{Louvain} and \textsc{VieClus} in terms of running time and memory consumption. Box plots depict the distribution of running time (left) and memory consumption (right). Note the logarithmic scale on the y-axis. Runtime for \textsc{VieClus} is primarily influenced by the tunable preset for evolutionary rounds, which is set to five minutes in our experiments. Instances where \textsc{Louvain} and \textsc{VieClus} failed are excluded for all algorithms.}

                \vspace*{-.125cm}
		\label{plot:boxplot}
\end{figure}

Compared to \textsc{Louvain} and \textsc{VieClus}, all modes of our algorithm consume less memory, as depicted in Figure~\ref{plot:boxplot}. \algname-Strong requires, on average, only 18.3\% of \textsc{Louvain}'s memory consumption and only 10.8\% of \textsc{VieClus}'s, while achieving over 96\% of the solution quality of both in-memory algorithms, on average. \algname-Light+ further optimizes this trade-off, retaining much of \algname-Strong’s quality benefits at a significantly reduced computational cost: on average, it is 55.8$\times$ faster than \textsc{VieClus}, uses only 3.0\% of \textsc{VieClus}’s memory, and retains 81.1\% of its solution quality; relative to \textsc{Louvain}, it runs 3.06$\times$ faster, consumes  5.0\% of \textsc{Louvain}’s memory, and still achieves 83.2\% of its quality, on average.
\begin{tcolorbox}
	\textbf{Observation 3.} \algname-Light is 2.6$\times$ faster and uses two-thirds of \textsc{Hollocou}'s memory requirement while offering 89.8\% better quality. \algname-Light+ maintains memory and runtime efficiency while more than doubling solution quality over \textsc{Hollocou}. \algname-Strong utilizes less than a fifth of the memory required by \textsc{Louvain} and \textsc{VieClus} with comparable solution quality.
	
\end{tcolorbox}

\subsection{Ground-Truth Community Recovery}
To further assess the effectiveness of our algorithm in practical settings and answer \textbf{RQ4}, we analyzed its performance on graphs with ground-truth communities using Normalized Mutual Information (NMI)~\cite{Lancichinetti_2009} as a measure of agreement between the predicted clusters and the ground-truth clusters. Our results confirm that \algname-Light outperforms \textsc{Hollocou} in ground-truth recovery, achieving a 7.0\% improvement over \textsc{Hollocou}, on average across all instances. \algname-Strong further enhances this accuracy, achieving a 16.8\% improvement over \textsc{Hollocou}, on average across all instances.
\begin{table}[t]
	\centering
	\caption{\textbf{Ground Truth Comparisons}. We compare modularity and Normalized Mutual Information (NMI) scores against ground truth community structures across various benchmark graphs. Higher values indicate better alignment with known community structures.}
	\resizebox{\textwidth}{!}{%
		\begin{tabular}{|l@{}cc|cc|cc|cc|cc|}
			\toprule
						\textsc{Graph} & \multicolumn{2}{c|}{\textsc{Light}} & \multicolumn{2}{c|}{\textsc{Light+}} & \multicolumn{2}{c|}{\textsc{Evo}} & \multicolumn{2}{c|}{\textsc{Strong}} & \multicolumn{2}{c}{\textsc{Hollocou}} \\
			& Modularity & NMI & Modularity & NMI & Modularity & NMI & Modularity & NMI & Modularity & NMI \\
			\midrule
			\rowcolor{shade1} Cora & 0.7158 & 0.3993 & 0.7434 & 0.3993 & 0.7820 & 0.4274 & \textbf{\textcolor{PineGreen}{\underline{0.7991}}} & \textbf{\textcolor{Mulberry}{\underline{0.4419}}} & 0.5372 & 0.3867 \\
			\rowcolor{shade2} Citeseer & 0.7926 & 0.3318 & 0.8108 & 0.3314 & 0.8781 & 0.3382 & \textbf{\textcolor{PineGreen}{\underline{0.8867}}} & \textbf{\textcolor{Mulberry}{\underline{0.3384}}} & 0.6730 & 0.3309 \\
			\rowcolor{shade1} AmazonCP & 0.5627 & 0.4397 & 0.6133 & 0.4679 & 0.5944 & 0.4507 & \textbf{\textcolor{PineGreen}{\underline{0.6231}}} & \textbf{\textcolor{Mulberry}{\underline{0.4784}}} & 0.0744 & 0.3404 \\
			\rowcolor{shade2} PubMed & 0.6388 & 0.1658 & 0.6775 & 0.1692 & 0.7301 & 0.1871 & \textbf{\textcolor{PineGreen}{\underline{0.7552}}} & \textbf{\textcolor{Mulberry}{\underline{0.1917}}} & 0.3283 & 0.1691 \\
						\midrule
			\textsc{GEOMEAN} & 0.6720 & 0.3135 & 0.7074 & 0.3199 & 0.7388 & 0.3323 & \textbf{\textcolor{PineGreen}{\underline{0.7599}}} & \textbf{\textcolor{Mulberry}{\underline{0.3422}}} & 0.3065 & 0.2930 \\
						\bottomrule
			\end{tabular}%
				}
			\label{tab:modularity_nmi}
		\end{table}

\section{Conclusion}
\label{sec:conclusion}
To address the need for a high-quality graph clustering algorithm with low memory overhead, we propose \algname, a graph \textbf{Clu}stering algorithm in a \textbf{St}reaming setting with multi-stage refinement using \textbf{R}e-streaming and \textbf{E}volutionary heuristics. \algname~processes the graph in a node stream, dynamically constructing a quotient graph to efficiently refine clustering and optimize modularity using global information. Through memetic optimization and re-streaming with local search, the algorithm iteratively adjusts the clustering until modularity converges to the local optimum. Our results establish the superiority of \algname~over the state-of-the-art streaming clustering approach, \textsc{Hollocou}, in solution quality, ground-truth community recovery, runtime and memory efficiency. \algname-Strong sets a new benchmark, delivering a 150\% improvement in solution quality over \textsc{Hollocou} on average across all instances, making it the best-performing streaming clustering method. Even our lightest configuration, \algname-Light significantly outperforms \textsc{Hollocou}, improving solution quality by 90\% while being faster and using less memory. Additionally, our algorithm bridges the solution quality gap between streaming and in-memory clustering algorithms, achieving comparable solution quality to in-memory algorithms, while requiring less than a fifth of their memory consumption, on average. Notably, \algname~improves ground-truth recovery over \textsc{Hollocou} by up to 17\% on average, making it the best-performing streaming algorithm for accurately recovering known community structures. These outcomes highlight the considerable potential of our algorithm, as well as the versatility of our configurations, positioning it as a dynamic and promising tool for high-quality clustering even in resource-constrained settings. 
\vfill

\bibliography{arxiv_clean}
\vfill
\appendix
\section{Appendix: Modularity Equivalence of Quotient Graph}

\begin{theorem}
	\label{thm:modularity_invariance}
	Let $G = (V, E)$ be an undirected graph with edge weights $w: E \to \mathbb{R}_{>0}$, and let $\mathcal{C} = \{C_1, C_2, \dots, C_k\}$ be a partition of $V$ into clusters. The modularity of $\mathcal{C}$ in $G$ is:
	\begin{equation}
		Q_G(\mathcal{C}) = \frac{1}{m} \sum_{C_i \in \mathcal{C}} \left( K_{C_i \to C_i} - \frac{\text{vol}(C_i)^2}{2m} \right),
	\end{equation}
	where:
	\begin{itemize}
		\item $m = \sum_{(u,v) \in E} w(u,v)$ is the total edge weight,
		\item $K_{C_i \to C_i} = \sum_{u, v \in C_i, (u,v) \in E} w(u,v)$ is the total intra-cluster edge weight,
		\item $\text{vol}(C_i) = \sum_{v \in C_i} d_w(v)$ is the volume of $C_i$, where $d_w(v) = \sum_{u \in N(v)} w(v,u)$ is the weighted degree.
	\end{itemize}
	
	Consider the quotient graph $G_Q = (V_Q, E_Q)$, where each cluster $C_i$ in $G$ is contracted into a supernode $v'_i$, with edge weights:
	\begin{gather}
		w'(v'_i, v'_j) = K_{C_i \to C_j} = \sum_{u \in C_i, v \in C_j, (u,v) \in E} w(u,v),
	\end{gather}
	and self-loop weights $w'(v'_i, v'_i) = K_{C_i \to C_i}.$ Then:
	\begin{enumerate}
		\item Define the clustering $\mathcal{C}' = \{ \{ v'_i \} \mid C_i \in \mathcal{C} \}$ in $G_Q$. The modularity of the clustering $\mathcal{C}'$ of the quotient graph $G_Q$ satisfies $Q_{G_Q}(\mathcal{C}') = Q_G(\mathcal{C}).$
		\item Given any clustering $\mathcal{C}'$ in $G_Q$, its modularity is preserved when expanded to $G$, i.e., if $\hat{\mathcal{C}}$ is the corresponding clustering in $G$, then $Q_G(\hat{\mathcal{C}}) = Q_{G_Q}(\mathcal{C}')$.
	\end{enumerate}
\end{theorem}

\begin{proof}
	The total edge weight in $G_Q$ remains equivalent by construction:
	\begin{gather}
		m' = \sum_{(v'_i, v'_j) \in E_Q} w'(v'_i, v'_j) = \sum_{(u,v) \in E} w(u,v) = m.
	\end{gather}
	Since we defined $\mathcal{C}' = \{ \{ v'_i \} \mid C_i \in \mathcal{C} \}$ in $G_Q$, we get that the modularity of $\mathcal{C}'$ in $G_Q$ is:
	\begin{gather}
		Q_{G_Q}(\mathcal{C}') = \frac{1}{m} \sum_{v'_i \in V_Q} \left( w'(v'_i, v'_i) - \frac{\text{vol}(v'_i)^2}{2m} \right).
	\end{gather}
	where, by construction,
	$w'(v'_i, v'_i) = K_{C_i \to C_i}$ and $\text{vol}(v'_i) = \text{vol}(C_i)$ due to the weighted self-loops.
	Thus, we get:
	\begin{align}
		Q_{G_Q}(\mathcal{C}') &= \frac{1}{m} \sum_{C_i \in \mathcal{C}} \left( K_{C_i \to C_i} - \frac{\text{vol}(C_i)^2}{2m} \right) = Q_G(\mathcal{C}).
	\end{align}
	\(\therefore\) Modularity is invariant under contraction.
	
	For the reverse direction, consider a clustering $\mathcal{C}'$ in $G_Q$. The expansion of $\mathcal{C}'$ into $G$ defines a clustering $\hat{\mathcal{C}}$, where each supernode $v'_i \in C'_j$ is replaced by its original cluster $C_i$, forming $\hat{\mathcal{C}}_j = \bigcup_{v'_i \in C'_j} C_i$. The modularity of $\hat{\mathcal{C}}$ in $G$ is:
	\begin{gather}
		Q_G(\hat{\mathcal{C}}) = \frac{1}{m} \sum_{\hat{C}_j \in \hat{\mathcal{C}}} \left( K_{\hat{C}_j \to \hat{C}_j} - \frac{\text{vol}(\hat{C}_j)^2}{2m} \right).
	\end{gather}
	By construction, $K_{\hat{C}_j \to \hat{C}_j} = K_{C'_j \to C'_j}$ and $\text{vol}(\hat{C}_j) = \text{vol}(C'_j)$ $\implies$ $Q_G(\hat{\mathcal{C}}) = Q_{G_Q}(\mathcal{C}').$
	
	\(\therefore\) Modularity is invariant under expansion.
\end{proof}

\end{document}